\documentclass[runningheads]{llncs}
\usepackage{url}
\usepackage{comment}
\usepackage{multicol}
\usepackage{hyperref}
\usepackage{stfloats}
\usepackage{graphicx}
\usepackage{subfigure}
\usepackage{amsfonts}
\usepackage{bbm}
\usepackage{amsmath}
\usepackage{cases}
\usepackage{multirow}
\usepackage{diagbox}

\usepackage[ruled,vlined, linesnumbered]{algorithm2e}
\usepackage{tikz}
\usepackage{booktabs}
\usetikzlibrary{positioning}
\usepackage{ragged2e}
\usepackage{color,soul}
\usepackage{array}
\usepackage{chngpage}
\usepackage{multicol}
\usepackage{caption}
\usepackage{enumitem}

\newcommand{\DB}{\ensuremath{\mathcal{D}}}

\begin{document}
%
\title{Station-to-User Transfer Learning: \\
Towards Explainable User Clustering Through Latent Trip Signatures Using Tidal-Regularized Non-Negative Matrix Factorization}
%
%

\author{Liming Zhang\inst{1} \and
Andreas Züfle\inst{1} \and
Dieter Pfoser\inst{1}}
\institute{George Mason University, Fairfax VA 22030, USA
\email{\{lzhang22,azufle,dpfoser\}@gmu.edu}}

%
\authorrunning{Zhang et al.}
%
%
\maketitle              
\begin{abstract}
Urban areas provide us with a treasure trove of available data capturing almost every aspect of a population's life. This work focuses on mobility data and how it will help improve our understanding of urban mobility patterns.
Readily available and sizable farecard data captures trips in a public transportation network. However, such data typically lacks temporal modalities and as such the task of inferring trip semantic, station function, and user profile is quite challenging. As existing approaches either focus on station-level or user-level signals, they are prone to overfitting and generate less credible and insightful results. 
To properly learn such characteristics from trip data, we propose a Collective Learning Framework through Latent Representation, which augments user-level learning with collective patterns learned from station-level signals.
This framework uses a novel, so-called Tidal-Regularized Non-negative Matrix Factorization method, which incorporates domain knowledge in the form of temporal passenger flow patterns in generic Non-negative Matrix Factorization.
To evaluate our model performance, a user stability test based on the classical Rand Index is introduced as a metric to benchmark different unsupervised learning models. 
We provide a qualitative analysis of the station functions and user profiles for the Washington D.C. metro and show how our method supports spatiotemporal intra-city mobility exploration.
\end{abstract}
\section{Introduction}
\label{sec:intro}
Traffic conditions in urban areas across the world remain a global challenge. According to the 2019 INRIX Traffic Scorecard \cite{reed2019inrix}, people in large cities across the world, such as Rio de Janeiro, Paris, and Chicago waste an average of more than 150 hours stuck in traffic, wasting hundreds of billions of USD and creating unnecessary greenhouse gas emissions. With two thirds of the population living in urban areas by 2050 \cite{un18}, our future is characterized by mega cities in which urban mobility becomes a critical concerns.
For these reasons and others, many recent studies have focused on modeling and predicting human mobility in urban scenarios as surveyed in~\cite{wang2019urban}.

Paramount to improving urban mobility is to understand the purpose of trips of people~\cite{wang2017human}. Although the idea to directly collect trip purpose data through travel survey is a practice with a long history~\cite{nitsche2014supporting}, ubiquitous computing and crowdsourcing data \cite{wang2017human,furletti2013inferring,gong2012exploring,gan2018you} provide a complementary way beside conventional time-consuming and costly field surveying. Yet, a data driven approach is challenging, as available trip information does not specify the purpose of a trip.

\begin{figure}[!t]
\vspace{-1\baselineskip}
    \centering
    \includegraphics[width=\linewidth, trim={0cm 8.5cm 0cm 3cm}, clip]{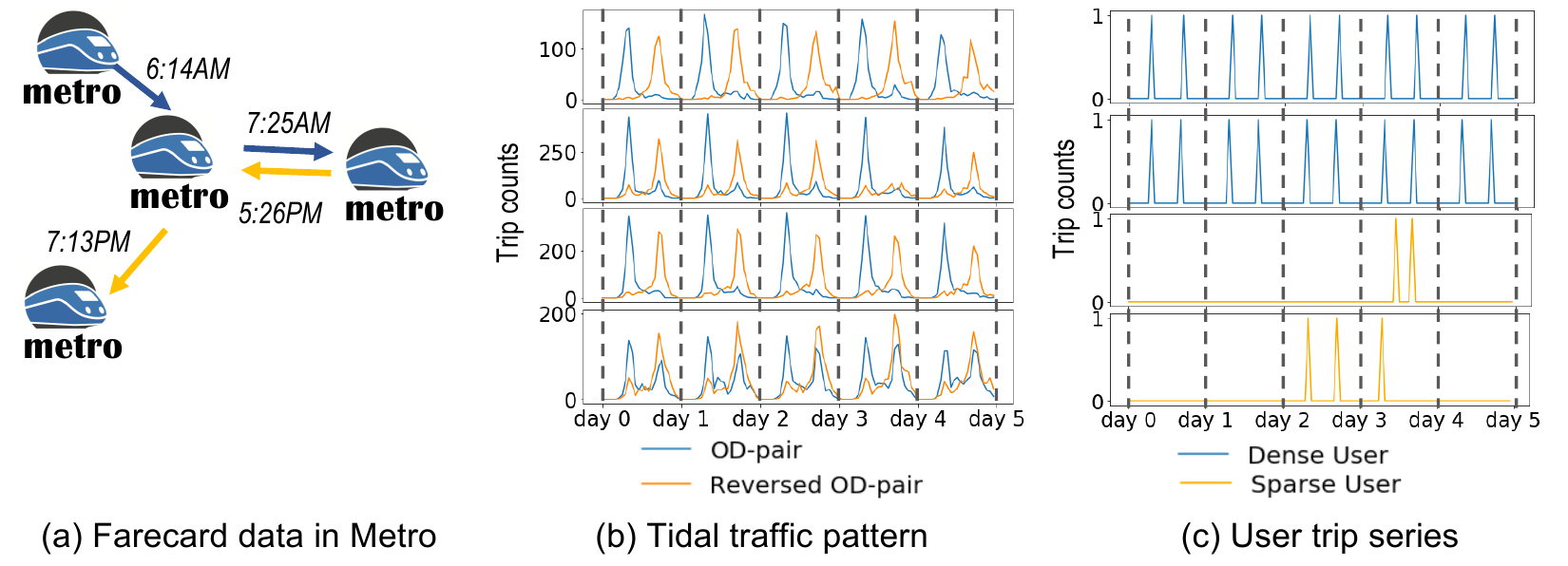}
    
    \caption{(a) A toy example of metro network. Each arrow indicates a trip with the timestamp; (b) Tidal traffic patterns observed within OD-pairs and associated reversed OD-pair (defined in Section \ref{sec:probdef}); (c) dense users have strong recurrent pattern of commuting trips, while sparse users have weaker commuting pattern and are hard to be learned.}
    \label{fig:real_example}
\end{figure}

The goal of this work is to tackle this challenge by providing data-driven methods to learn the purpose of trips.
More specifically, this work focuses on utilizing public transport data like Metro farecard data for urban mobility analysis. Millions of trip records in these Metro farecard data (Figure \ref{fig:real_example}) provide data including: Card ID, Entry station, arrival stations, entry time, and arrival time.
%
%
This data poses a number of unique challenges: (1) The purpose of a metro station is user-dependent, as the home-station of one user may be the work-station of another, and the recreation-station of yet another user; 
(2) a Card IDs do not injectively map to unique users, as one user may hold multiple cards including non-reusable temporary card, or multiple cards purchased at different time; (3) Trips may be missing, as user may choose to other modes of transport for certain trips.
These challenges blur the signal of each individual user.


%
To address these challenges, we propose a ``Station-to-User Transfer Learning'' framework based on a novel domain-specific ``Tidal-Regularized Non-Negative Matrix Factorization (TR-NMF)'' machine learning model. This framework defines similarities of users by mapping them into a latent station feature space. Creation of this feature space exploits knowledge about ``tidal'' behavior of users having recurrent morning and evening peaks~\cite{becker2013human}. We also propose first-of-its-kind clustering stability test as a cross-model evaluation metric to promote future benchmarking in station and user clustering researches.



The remainder of this paper is organized as follows: After surveying related work in Section \ref{sec:related}, we introduce the used datasets and formalize the problem of explainable user clustering in Section~\ref{sec:probdef}. Section \ref{sec:framework} introduces our new Station-to-User transfer learning framework with its novel Tidal-Regularized Non-Negative Matrix Factorization to achieve explainable clustering based on trip semantics, and a clustering stability test metric. Next, in Section \ref{sec:exper}, we provide both quantitative and qualitative evaluations of our approach. Finally, in Section \ref{sec:conclusion}, we emphasize our contributions and future direction of works.

\section{Related work}
\label{sec:related}
Early works on metro farecard data focus on descriptive statistics to characterize tidal pattern and dominant stations~\cite{liu2009understanding,gong2012exploring}. To learn the function of region in stations, 
Solutions have been proposed to infer the function of regions of a city based on individual mobility data in  ~\cite{yuan2012discovering}. This work uses topic modeling to map point of interest and user visits of a region to latent topics. These latent topics that are leveraged to assess similarity between regions. Following this approach, it has been shown in~\cite{zhang2013importance} that the function of a region changes over time, and that it is paramount to consider temporal dynamics. Specifically using Smart Card data, latent factor based solutions have shown capable of recognizing daily patterns, such as weekdays, weekends, and holidays~\cite{yang2017daily}. 

Related to our approach, a recent matrix factorization based approach to infer the temporal functions of regions (or stations) has been proposed in ~\cite{wang2016revealing}. This approach has been leveraged to identify tidal patterns of human mobility in~\cite{troia2017identification}. 
These works have in common that their goal is to infer the function of regions or stations. Our goal is to go a step further and to identify the ``function'' or signature of individual users, to assess the similarity of users to cluster them into groups of similar types of users.
Towards inferring user-specific activities, solutions have been proposed in~\cite{furletti2013inferring} using trajectory data and stop points. However, using only origin, destination, and time information available in metro farecard data, it is not possible to infer stops at specific points of interest to directly infer the purpose of a trip. 
Non-negative Matrix Factorization (NMF) based solutions have also been proposed for other problem related to urban mobility, such as predicting road traffic~\cite{han2016analysis,wang2014travel} and predicting metro traffic demand~\cite{duan2018understanding}. These works provide powerful solutions to predict traffic, but lack explanation of patterns. 
To capture spatial and temporal mobility patterns, existing works \cite{gong2012exploring,poussevin2015mining,tonnelier2016smart} use NMF to explain temporal patterns in daily life, such as commuting pattern that concentrates on morning and evening, and explain the function of urban spatial urban areas. 
In a recent work by Wang et al. \cite{wang2019understanding}, a context-aware tensor decomposition is used to explain urban mobility over space and time using a tensor factorization approach. 
These works have in common that they allow to model similar spatial and temporal urban dynamics, such as days having similar mobility patterns and regions having similar function. However, these approaches do not allow to assess similarity among individual users and passengers. In contrast, our approach unwraps the signatures of individual metro users, allowing to cluster similar users to explain individual users and the purpose of their trips.

\section{Problem Definition}\label{sec:probdef}

In this section, basic definitions and problem setup are presented to formalize our problem.


\begin{definition}[Trip Database]\label{def:DB}
Let $\mathcal{U}$ be a set of metro users, let $\mathcal{S}$ be a set of metro stations, $\mathcal{OD}=\mathcal{S}\times\mathcal{S}$ denote the set of all origin-destination station pairs, and let $\mathcal{T}$ be a set of time intervals or epochs. A trip database $\DB$ is a collection of tuples $(u,(o,d),t)\in \mathcal{U}\times\mathcal{S}\times\mathcal{S}\times\mathcal{T}$, where $u\in\mathcal{U}$ is a user, $(o, d) \in \mathcal{OD}$ is an OD-pair,  $o\in\mathcal{S}$ is the origin station, $d\in\mathcal{S}$ is the destination station,$t\in\mathcal{T}$ is the start time of the trip. 
\end{definition}

With trip database, we can define temporal flow matrix for all OD-pairs which aggregate trip data and stores the number of trips, grouped by OD-pairs, for each time epoch.
\begin{definition}[OD-pair Temporal Flow Matrix]\label{def:V} is denoted as
$\pmb{V} \in \mathbb{R}^{|\mathcal{OD}| \times |\mathcal{T}|}$ matrix, such that:
$$
\pmb{V}_{(o,d)\in\mathcal{OD},t\in\mathcal{T}}= |\{x\in\DB|x.o=o \wedge x.d=d \wedge x.t=t \}|
$$
\end{definition}

We further define a temporal flow matrix for each user, which aggregates the number of trips per user grouped by time epochs oblivious of stations.
\begin{definition}[User Temporal Flow Matrix]\label{def:U} is denoted as $\pmb{U} \in \mathbb{R}^{|\mathcal{U}| \times |\mathcal{T}|}$, matrix such that:
$$
\pmb{U}_{u\in\mathcal{U},t\in\mathcal{T}}=|\{x\in\DB|x.u=u \wedge x.t=t\}|
$$
\end{definition}
Given a OD-pair temporal flow matrix $\pmb{V}$ and a user temporal flow matrix $\pmb{U}$, \textbf{our problems of clustering users are as follows: separate users to different groups that maximize the internal similarity of travel.}

\section{Station-to-User (S2U) Transfer Learning Framework}
\label{sec:framework}
To better cluster users based on the purpose of their trips, we propose a framework to learn the temporal signature between stations and users in a collective manner. 
%
The diagram of this Station-to-User (S2U) Learning Framework is shown in Figure \ref{fig:framework} and has three main steps.

\begin{figure}[!t]
\vspace{-1\baselineskip}
    \centering
    \includegraphics[width=\linewidth, trim=0cm 12.8cm 8.5cm 3cm, clip]{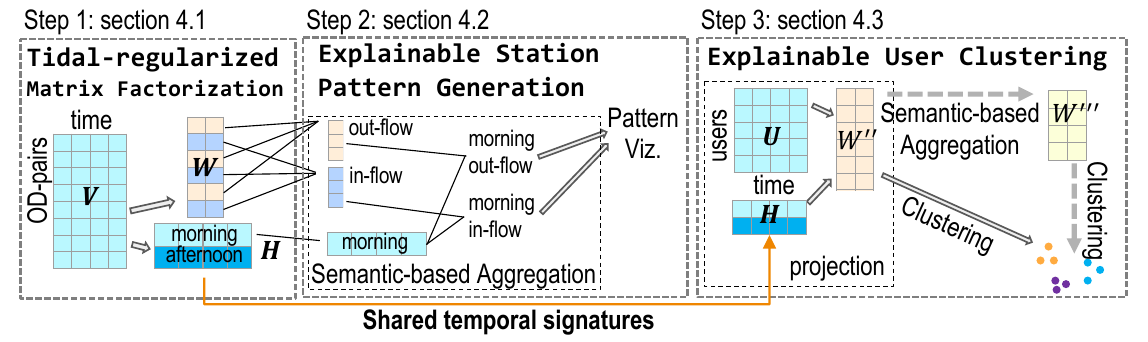}
    \vspace{-1\baselineskip}
    \caption{Station-to-User (S2U) Learning Framework}
    \label{fig:framework}
\end{figure}

\textbf{Step 1: Tidal-regularized matrix factorization:} Factorization of the OD-pair temporal flow matrix $\pmb{V}$ (c.f. Definition~\ref{def:V}) to find temporal latent features $\pmb{H}$ and latent trip features $\pmb{W}$.
To obtain interpretable features, we employ a tidal-regularized loss function to better fit the (empirically grounded) tidal pattern in urban mobility.
More details on this matrix factorization approach are found in Section~\ref{sec:tidal_loss}.

\textbf{Step 2: Decomposing Tidal Features:} Semantic-based Aggregation of latent trip features $\pmb{W}$ to get tidal features of in- and out-flow weights $\pmb{W}'$ for each station station.
The weights indicate, for example, the degree to which a station is a work destination (inflow) or a home destination (outflow).
Details of this approach are described in Section \ref{sec:cluster_station}.

\textbf{Step 3: Projection and clustering of user:} Mapping temporal behavior of users $\pmb{U}$ into the space of tidal features $H$. 
This yields a matrix $\pmb{W}''$ containing the tidal features of each user. The reason for mapping users into the station space is that tidal features of stations are more stable and less noisy as shown by our experimental evaluation. 
This approach allows to provide explainable behavioral difference between users, even for users with only a few observed trips, thus making user clustering results more stable and informative. 
More details of this step are found in Section \ref{sec:cluster_user}.

\subsection{Tidal-regularized Non-negative Matrix Factorization (TF-NMF)}
\label{sec:tidal_loss}

We decompose matrix $\pmb{V}\in \mathbb{R}^{|\mathcal{OD}| \times |\mathcal{T}|}$ into two non-negative matrices $\pmb{W}\in \mathbb{R}^{|\mathcal{OD}| \times K}$ and $\pmb{H}\in \mathbb{R}^{K \times |\mathcal{T}|}$, such that
$$
\pmb{V}\approx \pmb{W}\pmb{H}=:\hat{\pmb{V}},
$$
where $K$ is a positive integer, $\mathcal{OD}$ is the set of origin-destination pairs, and $\mathcal{T}$ is the set of temporal epochs (c.f. Definition~\ref{def:DB}).
To find $\pmb{W}$ and $\pmb{H}$ we minimize a loss function $\mathcal{L}$ defined by the mean square approximation error and the $L_1$ and $L_2$ norms of $\pmb{W}$ and $\pmb{H}$ as follows~\cite{hoyer2004non}:
$$
\mathcal{L} = \sum_i \sum_t (\pmb{V}_{i,t} - \hat{\pmb{V}}_{i,t})^2 
+ \alpha \eta  ( \|W\|_1 + \|H\|_1 )
+ \alpha (1 - \eta) ( \|W\|_2 + \|H\|_2)
$$, where $\|\cdot\|_1$ is the $L_1$ norm of a matrix, $\|\cdot\|_2$ is $L_2$ (or Frobenius norm) of a matrix, and $\alpha, \eta$ are hyper-parameters. 

Motivated by a \textit{tidal traffic pattern} observed in urban areas (cf. \cite{taylor2000network,alvizu2016energy,troia2017identification}), we observe that the tidal-traffic pattern has strong temporal peaks, as the morning commute happens before $11am$, while the reverse afternoon commute happens after $2pm$. 
%

We incorporate this \textit{a-priori knowledge} into our NMF approach by adding \textbf{a tidal-regularized (TR) loss to the generic NMF loss function.} It acts as a soft regularization to guide learned temporal signatures towards a better fit to such a tidal pattern. 
To understand the tidal regularized loss, we partition factor matrices $\pmb{W}$ and $\pmb{H}$ to separate tidal features corresponding to daily morning and evening peaks. This approach is illustrated in Figure \ref{fig:tidal_loss} and described as follows.

\textbf{(i) Grouping latent features by temporal semantics.} Generic NMF does not consider (or understand) temporal ordering, as temporal epochs (columns in $\pmb{U}$ and $\pmb{V}$) are treated as nominal (but not ordinal) variables. We sort latent features by their temporal semantics to understand and guide the learning process. 
We exploit that matrix $\pmb{H}$ provides the temporal semantics of each latent feature: It describes each temporal epoch (such as each hour), by $K$ latent features. Assuming tidal patterns, we expect some latent features to have larger weights in the morning epochs, and some latent features to have larger weights in the evening epochs. As an example in the upper right of Figure~\ref{fig:tidal_loss}, the scatter plot shows the temporal semantics of $6$ latent features Washington D.C. metro data. We observe that for latent feature $1$ and $2$, we clearly observe morning hour semantics. For features $5$ and $6$, the semantic feature is on the afternoon hours. Feature $3$ are not clearly delineated between morning and afternoon. Feature $6$ describe late afternoon and evening semantics.

We swap lines in matrix $\pmb{H}$ such that the first $k\leq K$ feature correspond to the morning features, and the last $k^\prime\leq K-k$ features correspond to the evening features.
To ensure that this swapping of columns in $\pmb{H}$ does not affect the factor product $\pmb{V}$, we perform the same swaps among lines of $\pmb{W}$ using the following observation.

\begin{lemma}\label{lem:swap}
Let $W\in R^{m\times K}$, $H\in R^{K\times n}$.
Further, let $x,y\leq K$, let $W^\prime$ be obtained by swapping lines $x$ and $y$ in $W$, and let $H^\prime$ be obtained by swapping columns $x$ and $y$, then
$$
WH = W^\prime H^\prime
$$
\end{lemma}
\begin{proof}
Let $V=WH$, and let $V^\prime=W^\prime H^\prime$. For any cell $v_{ij}$ in $V$ is derived by matrix multiplication as
$$
v_{ij}=\sum_{k=1}^K w_{ik}h_{ki} = \sum_{k=1,k\neq x,y}^K w_{ik}h_{ki} + w_{xk}h_{kx} + w_{yk}h_{ky}
$$
Equivalently, we obtain 
$$
v^\prime_{ij}=\sum_{k=1}^K w_{ik}h_{ki} = \sum_{k=1,k\neq x,y}^K w_{ik}h_{ki} + w_{yk}h_{ky} + w_{xk}h_{kx}.
$$
Since $w_{xk}h_{kx} + w_{yk}h_{ky}=w_{yk}h_{ky} + w_{xk}h_{kx}$ by commutativity of multiplication, we get $v_{ij}=v^\prime_{ij}$ for any $i,j\leq k$. Thus $V=V^\prime$. 
\end{proof}
Lemma~\ref{lem:swap} allows us to assume, without loss of generality, that columns of 
$\pmb{W}$ and lines of $\pmb{H}$ are grouped into morning features first and afternoon features last. 

\textbf{(ii) Grouping Symmetric OD-Pairs:} Symmetric origin-destination pairs $a,b \mathcal{OD}, \forall a = (u,v), b = (v,u)$ have opposite direction of tidal traffic flow. To make learned latent tidal features fitting to tidal traffic flow, we aims to minimize the difference between morning flow of $a$ and evening flow of $b$ which is in the symmetric origin-destination of $a$.
We drop all the OD-pairs with $u = v$, since it is meaningless for a user entering and exiting the same station.




\textbf{(iii) Temporal partitions}. Based on (i) and (ii), we can partition $\pmb{W}$ and $\pmb{H}$; five partitions for $\pmb{W}$ according to OD direction and temporal ordering of latent components as shown in Figure~\ref{fig:tidal_loss}, in which the white cells are weights of non-commuting signatures. 
%
Accordingly, there are five partitions for $\pmb{H}$ based on temporal ordering of latent components and time of the day, shown in Figure \ref{fig:tidal_loss}, in which the white cells are non-commuting signatures.
Figure \ref{fig:tidal_loss} also shows how the tidal-regularized loss is calculated by minimizing the differences between Reverse OD-Pairs' total morning and afternoon commute flows. For example, the reconstructed morning trip matrix $(u, v)$ is $\pmb{W}^{1} \pmb{H}^{1}$. Instead of calculating the difference for each $t$, we need the total accumulated morning trip flow for OD-pair $(u, v)$ as 
$\pmb{r}_{(u,v)}^{morning} = \sum_{t=1}^{t'}\Big(\pmb{W}_{(u, v),}^{1} \pmb{H}_{,t}^{1}\Big)$. 
Similarly, its Reverse OD-Pair $(u', v')$ the total afternoon flow as $\pmb{r}_{(u',v')}^{afternoon} = \sum_{t=t'+1}^{T}\Big(\pmb{W}_{(u', v'),}^{4} \pmb{H}_{,t}^{4}\Big)$. 
The next step is to add all Reverse OD-Pairs' differences, $\big(\pmb{r}_{(u,v)}^{morning} - \pmb{r}_{(u',v')}^{afternoon}\big)^2$. 
Given that the combinations $\pmb{H}^{3}$ and $\pmb{H}^{2}$ have as little flow as possible, we regularize them to (close to) zero. 
\begin{figure}
\vspace{-1\baselineskip}
    \centering
    \includegraphics[width=\linewidth, trim=0cm 7.5cm 4cm 3cm, clip]{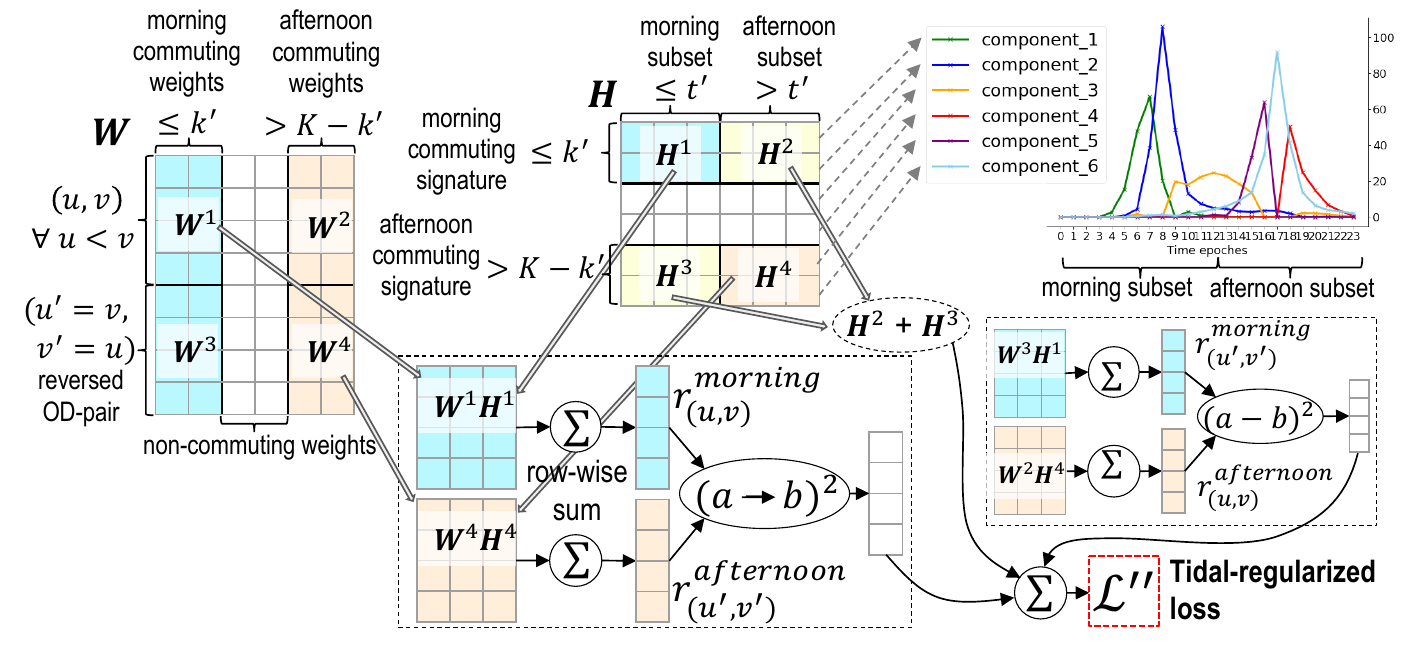}
    \vspace{-1\baselineskip}
    \caption{Partitions of weight matrix $\pmb{W}$ and latent component matrix $\pmb{H}$ by Reverse OD-Pairs and temporal ordering, and compositions of tidal-regularized loss}
    \label{fig:tidal_loss}
\end{figure}

In summary, the Tidal-regularized loss is formulated as follows.
\begin{equation}
\begin{split}
    \mathcal{L}'' = &\gamma \sum_{(u, v) = (1, 2)}^{u \leq |V|-1, \forall u < v}\big((\pmb{r}_{(u,v)}^{morning} - \pmb{r}_{(u',v')}^{afternoon})^2 + (\pmb{r}_{(u',v')}^{morning} - \pmb{r}_{(u,v)}^{afternoon})^2\big)\\
    &+ \rho \big(\|\pmb{H}^{3}\|_{F}^2 + \|\pmb{H}^{2}\|_{F}^2\big) \\
    & \pmb{r}_{(u,v)}^{morning} = \sum_{t=1}^{ t'}\Big(\pmb{W}_{(u, v),}^{1} \pmb{H}_{,t}^{1}\Big),\,\,
    \pmb{r}_{(u',v')}^{afternoon} = \sum_{t=t'+1}^{T}\Big(\pmb{W}_{(u', v'),}^{4} \pmb{H}_{,t}^{4}\Big) \\
    & \pmb{r}_{(u',v')}^{morning} = \sum_{t=1}^{t'}\Big(\pmb{W}_{(u', v'),}^{3} \pmb{H}_{,t}^{1}\Big),\,\,
    \pmb{r}_{(u,v)}^{afternoon} = \sum_{t=t'+1}^{T}\Big(\pmb{W}_{(u, v),}^{2} \pmb{H}_{,t}^{4}\Big)
\end{split}
\label{eq:tidal_loss}
\end{equation}

The total regularization loss $\mathcal{L}$ is defined as $\mathcal{L} = \mathcal{L}' + \mathcal{L}''$. We use Tensorflow to develop our new algorithm and directly utilize the \textit{Autodiff} features, which automatically compute the gradient update rules to optimize $W$ and $H$. 
To terminate training, we either use the Mean-Square-Error $\mathcal{L}'$ (setting a minimum threshold) or use a fixed number of training steps.
Post training, we convert the latent representation to a unit vector (cf. \cite{xu2003document}) as  follows:
$h_{k, t} = \frac{h_{k, t}}{\sqrt{\sum_t^T h_{k, t}^2}}, \,\,w_{i, k} = w_{i, k} \sqrt{\sum_t^T h_{k, t}^2}$

\subsection{Decompose modalities as station functions and clustering stations with explainable temporal modality}
\label{sec:cluster_station}
\textbf{Relating signatures to station functions:} 
To cluster stations, we need to convert the OD-pair flow matrix to in-flows and out-flow for each station.
To determine the function of a station (home, work, tourism, etc.) we examine the in- and out-flow of stations (cf. \cite{gong2012exploring,briand2015mixture}). 
In-flow symbolizes attracting people for, e.g., work, and we refer to this as ``attractivity'' function of a place. Out-flow indicates people leaving, e.g., from home or hotels, and we refer to this as ``generativity'' function of a place. 
However, different from existing works, we want to  distinguish not only between commuting and other functions, but also between different types of commuting (flexible work hours, etc.) for a station. 
We decompose station functions based on the meaning of temporal signatures $\pmb{H}$ identified by TR-NMF.

Finding explainable station functions is achieved by ``\textit{semantics-based aggregation}'' (cf. Figure \ref{fig:gene_attr}) that cherry-picks specific temporal signatures according to our partitions and their peak hours. We can select only one type of temporal signature and use its corresponding weights to recover the in- and out-flow of each station. 
Temporal signatures are used to explain trip semantics, such as morning  ($\pmb{H}^1$) and afternoon commutes ($\pmb{H}^4$) (cf. Figure~\ref{fig:tidal_loss}).
Furthermore, within both $\pmb{H}^1$ and $\pmb{H}^4$, there could be multiple rows of temporal signatures for different hours, e.g., $7am$, $8am$. 
Figure~\ref{fig:gene_attr} gives an example for temporal signatures peaking at $7am$ and aggregates the OD-pair weights to recover the in- and out-flow for stations. The recovered flows inherit the strong semantic meaning for different stations. Stations with large early hour in-flows are strong attractivity places, where stations with large out-flows are strong generativity places. Section \ref{sec:exp_qualitative} will give examples for  Washington D.C.
\begin{figure}[!t]
\vspace{-1\baselineskip}
  \centering
  \includegraphics[width=\linewidth, trim=0cm 7cm 5cm 3.5cm, clip]{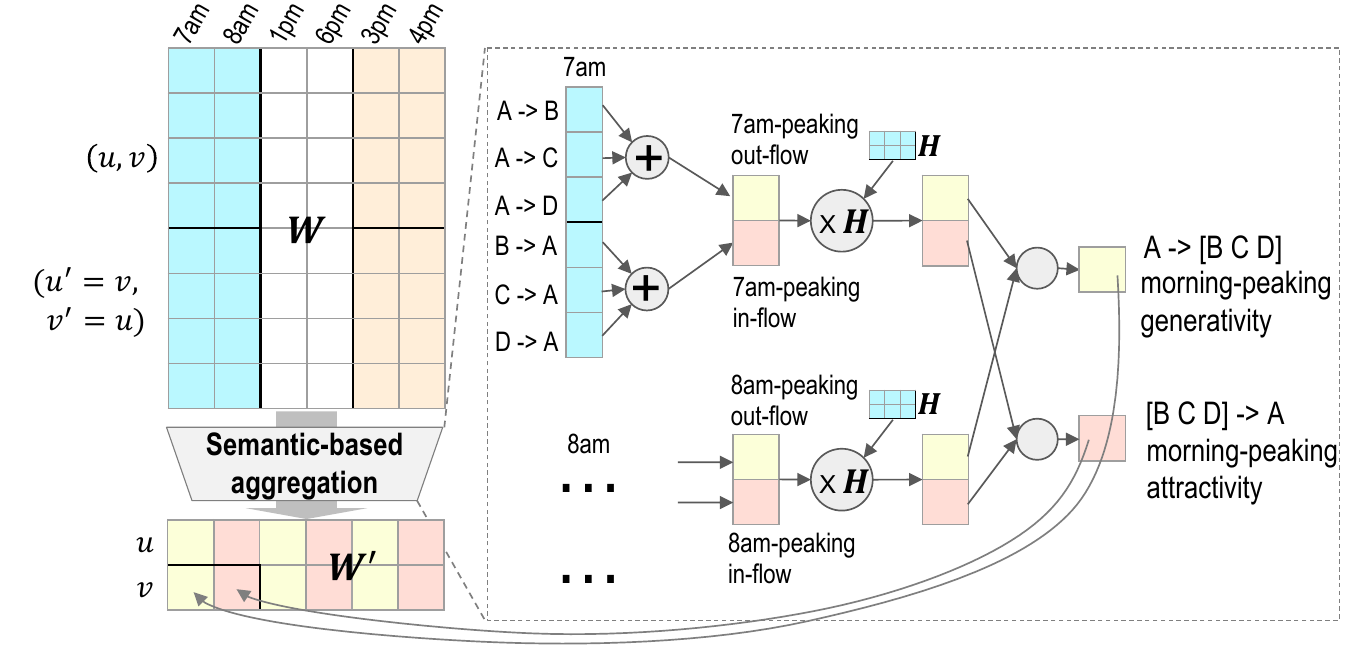}
  \vspace{-1\baselineskip}
  
  \caption{Semantics-based aggregation operation to get explainable generativity and attractivity station functions}
  \label{fig:gene_attr}
\end{figure}




\subsection{Learning user clusters}
\label{sec:cluster_user}

To find the projection of users is the same dimension reduction task as NMF. We implement a part of a multiplicative update rule to project a user supported by the learned temporal modalities:
\begin{equation}
    w''^{(i+1)}_{u,k} \leftarrow w''^{(i)}_{u,k} \frac{(\pmb{U}_{u,} \pmb{H}^T)_{u,k}}{(\pmb{W}^{(i)} \pmb{H} \pmb{H}^T)_{u,k}}
    \label{eq:W_update}
\end{equation}
where $w''_{u,k}$ are weights of shared temporal signature of a matrix $\pmb{W}''$ for user ID $u$ and temporal modalities $k$ (column). $i$ is the number of iterations needed until a convergence criterion (like Mean Square Error) is met. More details on multiplicative update rules can be found in, e.g., \cite{lee2001algorithms,xu2003document}.

After finding the weight matrix $\pmb{W}''$ for users, a clustering algorithm like k-Means++ \cite{arthur2006k} can be used to cluster them. Since each weight is associated with shared temporal signature, we can again use a semantics-based aggregation to obtain $\pmb{W}'''$ for a different explainable clustering.

\subsection{User clustering stability test}
\label{sec:stability_test}
To assess the performance of clustering methods, we introduce a novel domain-specific quantitative evaluation metric called ``Clustering Stability Test''. Although many different models are proposed in existing work using farecard data, no domain-specific quantitative evaluation metric is available to compare model performance. 
Different model assumptions and procedures prevent cross-model comparisons for clustering. 

Different metrics such as 
potential (sum of squared distances of samples to their closest cluster center) \cite{hastie2009elements}, 
log-likelihood score \cite{hastie2009elements}, perplexity score (information measure of generative probabilistic models) \cite{lucas2012mutual}, 
AIC \cite{hastie2009elements}, and 
BIC \cite{hastie2009elements}, 
are used to assess clustering quality based on model assumption or information theory but they are not able to judge the stability of a clustering. 
Various works exist to test the stability of a clustering, e.g.,  \cite{lange2004stability,rakhlin2007stability}. Our proposed metric is based on Adjusted Rand Index (ARI) \cite{hubert1985comparing}, a well-known measurement of the similarity between two clusterings with the same number of clusters $K$. 
For a set of data, like users $\mathcal{U}$, one clustering result assigns a set of group labels to each user with $X = \{ x_1, x_2, \dots, x_u\}$, while another clustering result assigsn a set of labels $Y = \{ y_1, y_2, \dots, y_u\}$. It can compute a ARI score $ARI_{x, y}$ based on these two label sets with random permutation of cluster label orders (cf. \cite{hubert1985comparing}). $ARI_{x, y}$ is a value with range $[-1, +1]$, where $0$ indicates complete random labeling, $+1$ stands for a perfect match, and  $-1$ indicates complete reversed labeling. 
However, generic $ARI$ only tells if two clustering sets are similar. It cannot tell which method is better in terms of stability, which is something our new method addresses.
\begin{figure}[htbp]
\vspace{-1\baselineskip}
    \centering
    \includegraphics[width=\linewidth, trim={0cm 13cm 7.5cm 3.5cm}, clip]{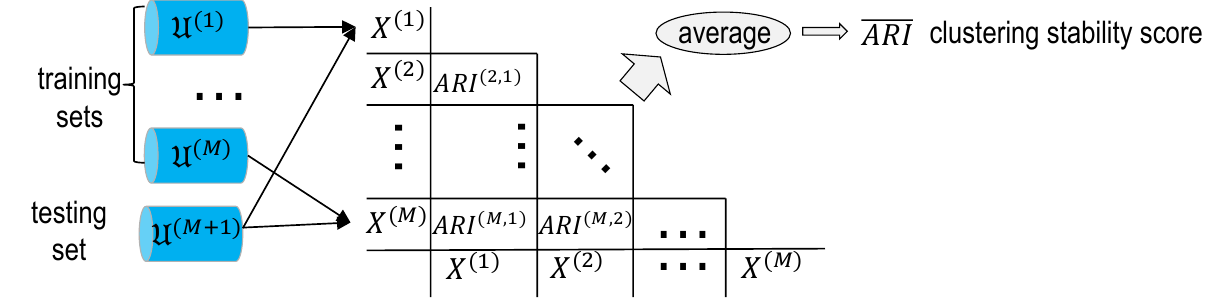}
    \vspace{-1\baselineskip}
    \caption{Clustering stability test}
    \label{fig:stability_test}
\end{figure}

Clustering stability in the context of semantic-poor farecard data aims to get a stable clustering, which is resilient to long-tail users who have abnormal or sparse behaviors, and without considering the internal processes. For example, we consider StoU framework as a whole method that output user labels, not just the internal k-means++ method. A stable clustering should capture those most significant behaviors without overfitting to those long-tail users. This means that if a large portion of users are unobserved, the user clustering labels do not change. So, our proposed procedure in Figure \ref{fig:stability_test} is similar to a typical classification problem that different training data are used to predict testing data. This approach is also inspired by some input randomization works \cite{lange2004stability,rakhlin2007stability}. We partition original user ID set $\mathcal{U}$ to non-overlapping $M$ training sets $\{\mathcal{U}^{(m)}\}_{m=1}^M$, and another non-overlapping testing set $\mathcal{U}^{(M+1)}$. Each method are applied to a mixed set that add a training set with the testing set, $\mathcal{U}'_m = \mathcal{U}^{(m)} + \mathcal{U}^{(M+1)}$. A label set $X^{(m)}$ can be got for $\mathcal{U}'_m$. For each pair of $X^{(m_i)}$ and $X^{(m_j)}$, we can get a ARI score $ARI^{(m_i, m_j)}$. Then, the average of all the paired ARI scores is used as the ``clustering stability score $\Bar{ARI}$'' to compare across different clustering methods. $\Bar{ARI}$ is going to be in range $[-1, +1]$, where $+1$ indicates a perfectly stable method.

\section{Experiments and Results}
\label{sec:exper}

In this section, we compare Collective Learning Framework using three real-world datasets and two competing methods. Section \ref{sec:exp_setting} introduces our experimenting settings. Evaluation of competing methods and our proposed method are compared using the new clustering stability metric in Section \ref{sec:exp_stability}. Then qualitative evaluations with spatial-temporal visualization are shown in Section \ref{sec:exp_qualitative} including explainable station clustering and user clustering results that decode urban mobility pattern in Washington D.C. as a case study.

\subsection{Experimental settings}
\label{sec:exp_setting}
\textbf{Real-world datasets:} we utilize three real-world farecard mobility datasets from Washington D.C.. Metro farecard data is from Washington Metropolitan Area Transit Authority (WMATA), which covers metropolitan area of Washington D.C. in one week from May-01-2016 to May-07-2016
Each fare record only contains limited information, which are an anonymous card ID, an entry station ID, an exit station ID, an entry timestamp, and an exit timestamp. Some preprocessing is done to convert timestamps to 24 hour time snapshot of a day. There are total of about 3.57 million trip records, and about 0.8 million unique user IDs. Taxi data is collected from different taxi agencies required by open-data initiative of Washington D.C. metropolitan []. It includes data from Washington D.C.'s bike-sharing serving with $401$ stations.
To preprocess these data, we convert raw timestamps to $24$ hours for $\mathcal{T}$. For metro data, we only keep sparsy users who have less or equal to 3 trips per week in this experiments. Since taxi data do not have stations, we use grid cells of $0.02$ degree (about $2Km$) by $0.02$ degree to build OD-pair temporal flow matrix. For metro, we do random selection without replacements to get $10$ training sets with $50,000$ users per set and $10,000$ users in testing set. For taxi, similarly, we get $8$ training sets with $10,000$ per set and $2,000$ users in testing set. For bike, we get $10$ training sets with $300$ users per set and $30$ users in testing set. Table \ref{tab:data_summary} is some descriptive summaries of our experiment datasets.
\begin{table}[!b]
    \centering
    \scriptsize
    \begin{tabular}{c|c|c|c|c|c}
        \toprule
         Data & Total users & Total trip & Training sets & Users per training & Users per testing \\
         \hline
         Metro &  516,976 & 845,700 & 10 & $50,000$ & $10,000$ \\
         \hline
         Taxi &  89,237 & 89,237 & 8 & $10,000$ & $2,000$ \\
         \hline
         Bike &  3,032 & 51,325 & 8 & $10,000$ & $2,000$ \\
        \bottomrule
    \end{tabular}
    \caption{Descriptive summaries of experiment datasets}
    \label{tab:data_summary}
\end{table}

\textbf{Metric:} we use the clustering stability test procedure to get the average or median $ARI$ scores of non-overlapping training sets. The higher $ARI$ score is, the better a method is. Because we introduce random splitting to get training sets, median (MED) and median absolute deviation (MAD) of a few dozens of runs are used to eliminate impacts from outlying cases.

\textbf{Competing methods:} 
Two competing methods are used with one controlled experiment method:
1) a naive model using raw trip flow of each time epoch as clustering features in KMeans++, noted as ``Naive'';
2) a baseline model using NMF on temporal trip counts feature proposed in \cite{carel2017non}, and apply KMeans++ clustering on reduced weights matrix, noted as ``NMF''; 
3) a controlled experiment for S2U framework which identically replicates a training set for clustering stability test, noted as ``Control''. Its goal is to show the effectiveness of S2U. 


\subsection{Quantitative comparisons with clustering stability test}
\label{sec:exp_stability}
What follows is a discussion of the performance of S2U framework compared to competing methods. Given the low complexity of Matrix Factorization and KMeans++, all the running times are around a few seconds.
\begin{table}[htbp]
    \centering
    \scriptsize
    \vspace{-1\baselineskip}
    \begin{tabular}{c|c|c|c|c|c}
    \toprule
        data  & ARI scores & Naive & NMF & S2U & Control \\
        \hline
        \multirow{2}{*}{Metro}  & [MED,MAD] of Mean & 0.5217, 0.0474 & 0.6501, 0.0342 & \textbf{0.7019, 0.0477} & 0.8034, 0.0590 \\
        \cline{2-6}
        & [MED,MAD] of Median & 0.5504, 0.0523 & 0.5815, 0.0333 & \textbf{0.6496, 0.0821} & 0.7347, 0.1400 \\
        \hline
        \multirow{2}{*}{Taxi} & [MED,MAD] of Mean & 0.5417, 0.0466 & 0.6605, 0.0804 & \textbf{0.8117, 0.0388} & 1.0, 0 \\
        \cline{2-6}
        & [MED,MAD] of Median & 0.4781, 0.0222 & 0.6079, 0.0239 & \textbf{0.8150, 0.0421} & 1.0, 0 \\
        \hline
        \multirow{2}{*}{Bike} & [MED,MAD] of Mean & 0.5727, 0.1308 & 0.5412, 0.1147 & \textbf{0.6347, 0.0836} & 0.7846, 0.0921 \\
        \cline{2-6}
        & [MED,MAD] of Median & 0.5697, 0.1293 & 0.5525, 0.1169 & \textbf{0.6272, 0.0844} & 0.7816, 0.1284 \\
        \bottomrule
    \end{tabular}
    \caption{Comparisons using user clustering stability test}
    \label{tab:user_stability}
\end{table}

Table \ref{tab:user_stability} contains performances of different methods. In each table cell, first value is median (MED) of a hundred of experiment runs, while second value is median absolute deviation (MAD) value in a hundred of runs. Fig. \ref{fig:clusters} shows distributions of clustering labels for each data and each model. 
The main finding is that our S2U outforms other two competing methods for all three data in both MED of Mean ARI and MED of Median ARI. Even if we subtract the MAD scores from MEDs (which is $95\%$ lower bound), S2U have a discounted lower bound of condifence close to two competing methods. For example of Metro, S2U have a subtracted value of $0.7019 - 0.0477 = 0.6542$, and NMF have a MED of $0.6501$. For taxi data, the gain is even larger with S2U's lower bound of $0.8117 - 0.0388 = 0.7729$ and NMF`s MED of Mean ARI of $0.6605$.
And, we are also confident in this result by examining Control model, which is consistantly much higher than normal S2U results.
\begin{figure}[!b]
\vspace{-1\baselineskip}
\subfigure[Naive for Metro]{\includegraphics[width = 0.3 \linewidth]{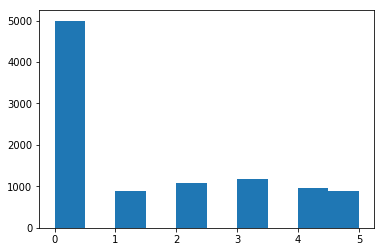}\label{fig:metro_naive}}
\subfigure[NMF for Metro]{\includegraphics[width = 0.3 \linewidth]{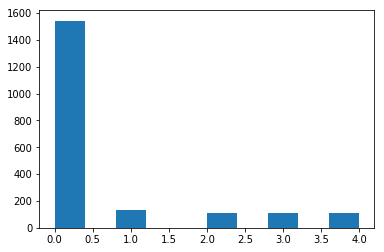}\label{fig:metro_nmf}}
\subfigure[S2U for Metro]{\includegraphics[width = 0.3 \linewidth]{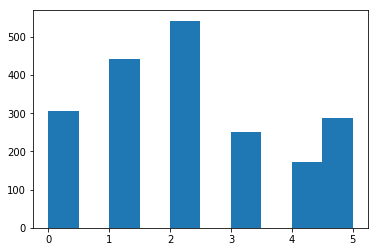}\label{fig:metro_s2u}}
\vspace{-1\baselineskip}

\subfigure[Naive for Taxi]{\includegraphics[width = 0.3 \linewidth]{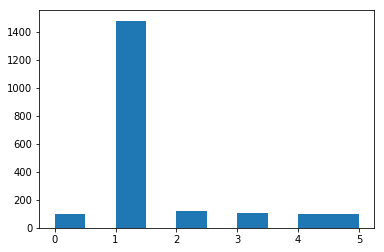}\label{fig:taxi_naive}}
\subfigure[NMF for Taxi]{\includegraphics[width = 0.3 \linewidth]{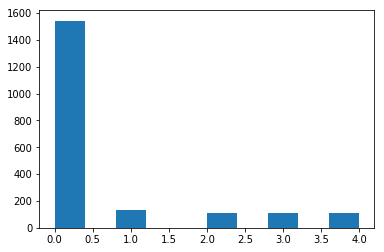}\label{fig:taxi_nmf}}
\subfigure[S2U for Taxi]{\includegraphics[width = 0.3 \linewidth]{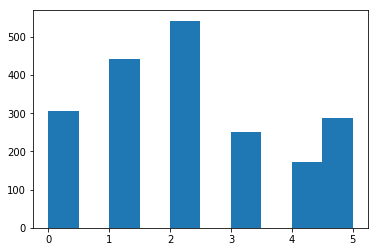}\label{fig:taxi_s2u}}
\vspace{-1\baselineskip}

\subfigure[Naive for Bike]{\includegraphics[width = 0.3 \linewidth]{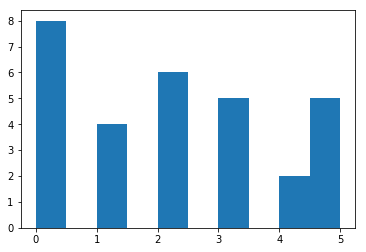}\label{fig:bike_naive}}
\subfigure[NMF for Bike]{\includegraphics[width = 0.3 \linewidth]{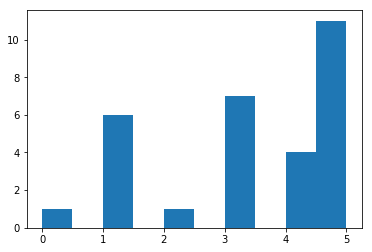}\label{fig:bike_nmf}}
\subfigure[S2U for Bike]{\includegraphics[width = 0.3 \linewidth]{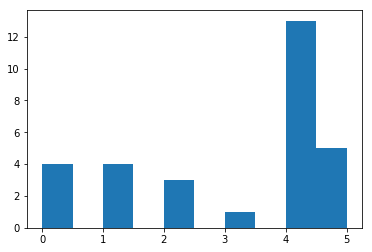}\label{fig:bike_s2u}}
\vspace{-1\baselineskip}

\caption{Histograms of User Clustering Labels (6 cluster number for all methods)}
\label{fig:clusters}
\end{figure}
Additional proof of a better clustering is the less skewed distribution of clustering labels in Figures \ref{fig:metro_s2u} and \ref{fig:taxi_s2u} compared to others of Naive and NMF models, which means that each cluster capture more meaningful patterns of sparsy users.. 
By checking clustering labels of Metro and Taxi with Naive and NMF in Fig. \ref{fig:metro_naive}, \ref{fig:metro_nmf}, \ref{fig:taxi_naive}, and \ref{fig:taxi_nmf}, it can be observed that there is a quite dominating cluster for both data. This is an indicator that these two methods do not capture the real patterns because raw user temporal flow matrix $\pmb{U}$ or decomposed $\pmb{U}$ by NMF model are not informative for these sparsy users. That is why we could not conclude that NMF is just as good as raw $\pmb{U}$ features for Metro, even though it gains $0.13$ of MED of mean ARI. For bike data, clustering labels are more evenly distributed in different groups. It is not quite changed for S2U and competing methods.

Looking at more details of the results, by comparing Naive model with NMF model, NMF already performs better for Metro and Taxi data with $0.13$ higher for MED of mean ARI scores and $0.03$ higher for MED of median ARI scores. If we consider the variance in random splitting, the improvement of median ARI is not significant since MAD of NMF's median ARI is $0.03$. But, the variance of NMF's mean ARI is only $0.03$, so this is a significant improvement if we use NMF model for user clustering compared to raw features.
But, NMF is not as good as Naive for bike data with about $0.01$ to $0.02$ lower on both MED of mean ARI scores and MED of median ARI scores. Of course, if we consider MADs of $0.1308$ and $0.1147$, the difference of $0.02$ is smaller than variance. It is not a significant gain, and it is hard to say NMF is better than Naive approach. 
For Taxi data, there is a huge improvement for cluster label distribution, while MED of mean ARI improve a lot by $0.15$ and MED of median ARI improve by a huge value of $0.21$. The possible reason is that Taxi data are already quite cluster-able and dominated by commuting patterns that our method fit into this commuting pattern strongly.

\subsection{Qualitative evaluations}
\label{sec:exp_qualitative}
A few demonstrations of how our framework compare to others are shown in this part. In follow Figure \ref{fig:latent_components}, temporal signatures found by TR-NMF model is shown. The $x$ axis is $24$ hours of time epoches of a day, and $y$ axis is the total signal in a time epoch. The left one \ref{fig:nmf_components} shows temporal signatures found by generic NMF algorithm, and the right one \ref{fig:trnmf_components} shows temporal signatures  found by TR-NMF. In both figures, components 1 \& 2 are morning commute signatures, components 5 \& 6 are afternoon commute signatures, and components 3 \& 4 are non-commuting signatures. We can see the overall trends are more or less the same, while the improvement of TR-NMF is between the component 4 (red one of non-commuting signature) and the component 6 (skyblue one of afternoon commute signature). Component 4 lost a few signal around noon, while those signal are used to improve component 6, because those temporal features are temporally closer to component 6's peaking feature. This result of TR-NMF definitely show how tidal-regularized loss constrain the learning, and provides better explainable power to temporal pattern in metro data.
\begin{figure}[hbtp]
\vspace{-1\baselineskip}
\subfigure[Temporal signatures by generic NMF]{\includegraphics[width = 0.5 \linewidth]{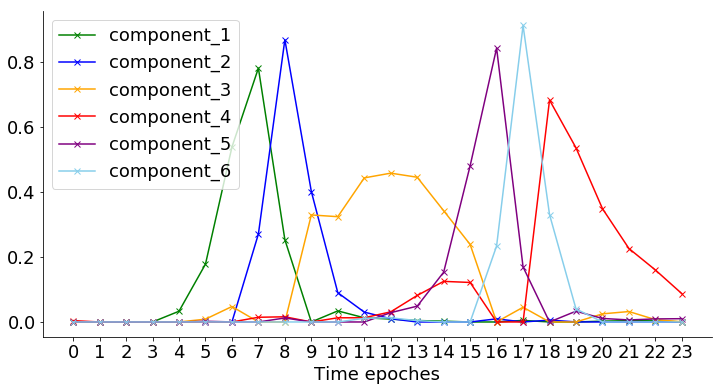}\label{fig:nmf_components}}
\subfigure[Temporal signatures by TR-NMF]{\includegraphics[width = 0.5 \linewidth]{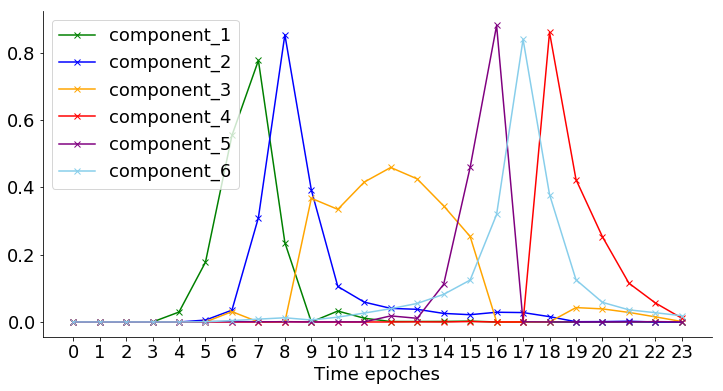}\label{fig:trnmf_components}}
\vspace{-1\baselineskip}

\caption{Temporal signatures by generic Non-negative Matrix Factorization (NMF) and Tidal-Regularized Non-Negative Matrix Factorization (TR-NMF))}
\label{fig:latent_components}
\end{figure}



\textbf{Explainable user clustering results}
t-Distributed Stochastic Neighbor Embedding (t-SNE) is a information-based machine learning visualization technique \cite{maaten2008visualizing}. It can reduce dimension through non-linear manifold while preserving informative similarity pattern within data. With its popularity to visualize high dimension data, we use it to qualitatively demonstrate the intrinsic properties of raw data itself (first raw of Figure \ref{fig:tsne}), and properties in transformed features by TR-NMF (second row of Figure \ref{fig:tsne}). The $x$ and $y$ axis are reduced two features. Each point is a user. Different colors of points are clusering labels found previously (first row is Raw model and second row is S2U model). First row is the output of t-SNE using raw features, and second row is the output of t-SNE using TR-NMF transformed features. By comparing raw features and transformed features, we can see that TR-NMF features are more informative with more clear clustering patterns. For Metro (first column), raw users are pretty flattened, while transformed users are more concentrated. It is a similar case for Bike (third column). However, Metro and Bike are both challenging problem themselves, since there are not strong clustering patterns in t-SNE transformed space. For Taxi, the raw features already contain strong similarity within several observable clusters, while TR-NMF did a good job to make those clusters condensed. Notice that there are more visually-appealing clusters in t-SNE space than S2U's clusters. But, for Naive model's clustering, blue point clustering includes most of the visually-appealing clusters. This is a minor problem for Taxi data since we can increase cluster number of S2U until the number reach a optimal value, however, increasing cluster number for Metro and Bike would not significantly improve clustering results.
\begin{figure}[htbp]
\centering
\vspace{-1\baselineskip}

\subfigure[Metro's Raw features]{\includegraphics[width=0.3\linewidth, trim={37cm 0cm 0cm 0cm}, clip]{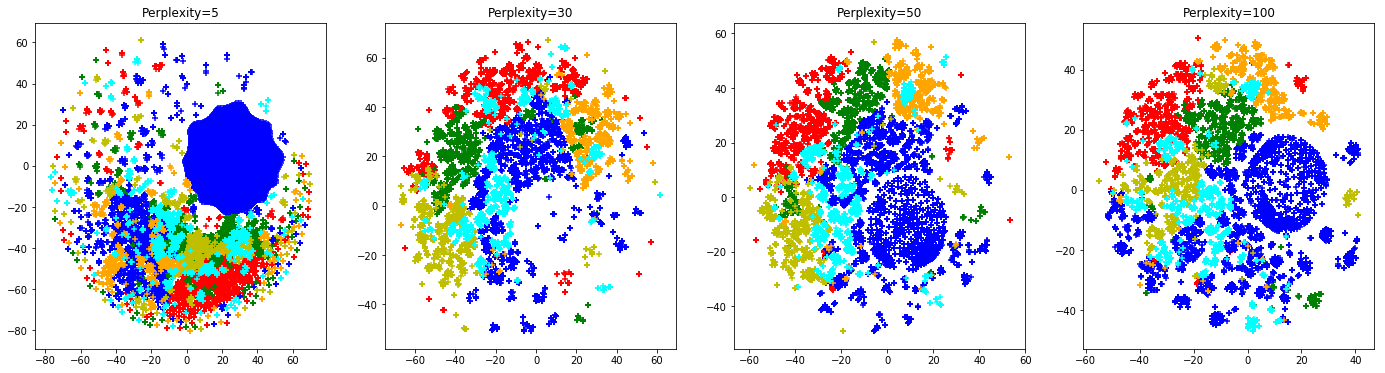}\label{fig:metro_raw_tsne}}
\subfigure[Taxi's Raw features]{\includegraphics[width=0.3\linewidth, trim={37cm 0cm 0cm 0cm}, clip]{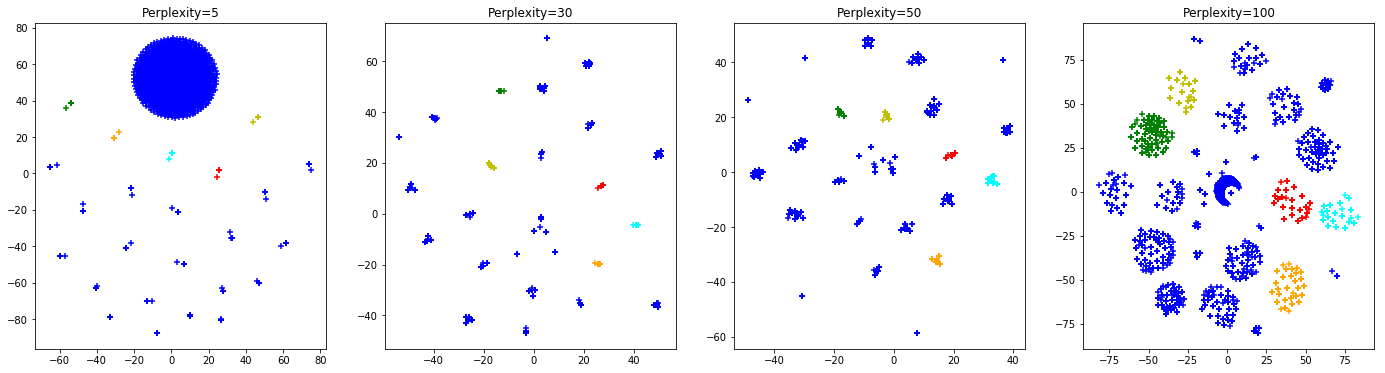}\label{fig:taxi_raw_tsne}}
\subfigure[Bike's Raw features]{\includegraphics[width = 0.3\linewidth, trim={37cm 0cm 0cm 0cm}, clip]{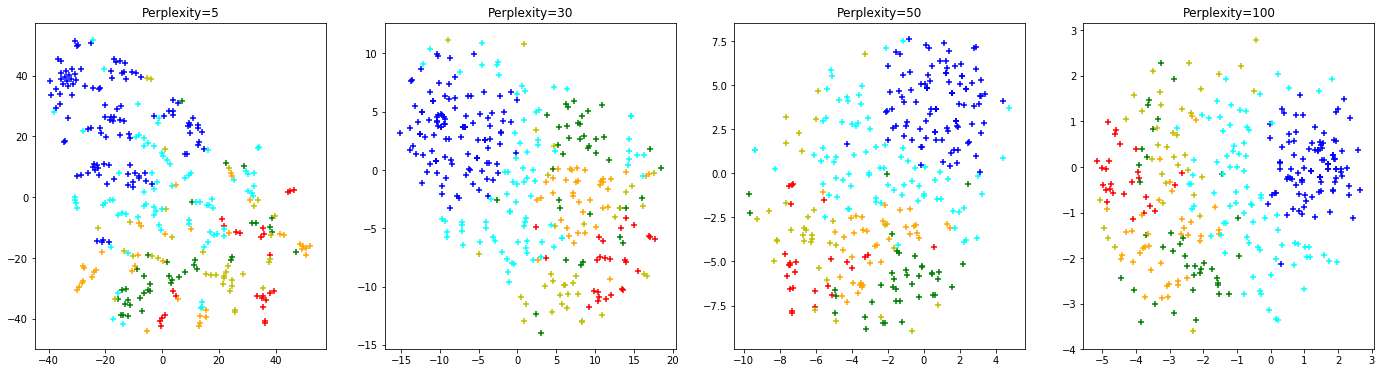}\label{fig:bike_raw_tsne}}
\vspace{-1\baselineskip}

\subfigure[Metro with S2U]{\includegraphics[width=0.3\linewidth, trim={37cm 0cm 0cm 0cm}, clip]{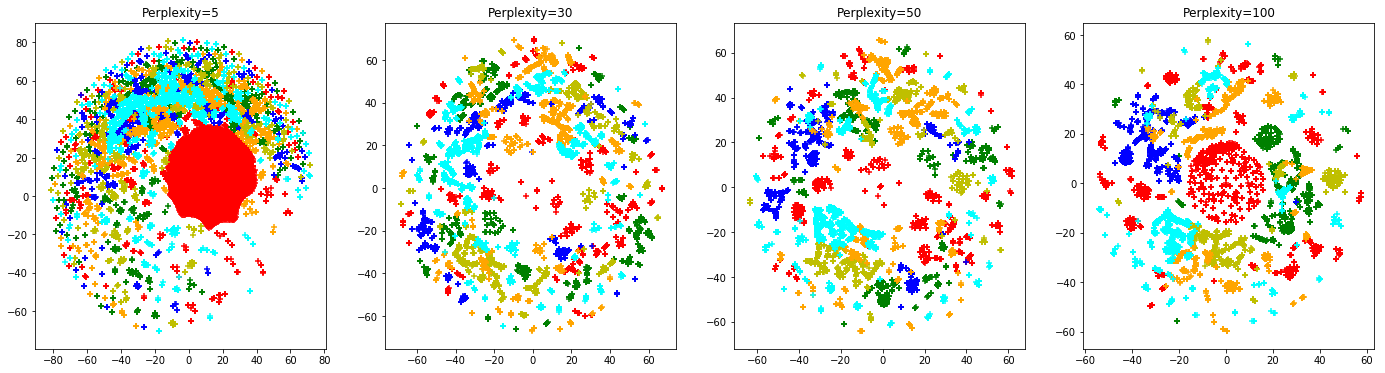}\label{fig:metro_tsne}}
\subfigure[Taxi with S2U]{\includegraphics[width=0.3\linewidth, trim={37cm 0cm 0cm 0cm}, clip]{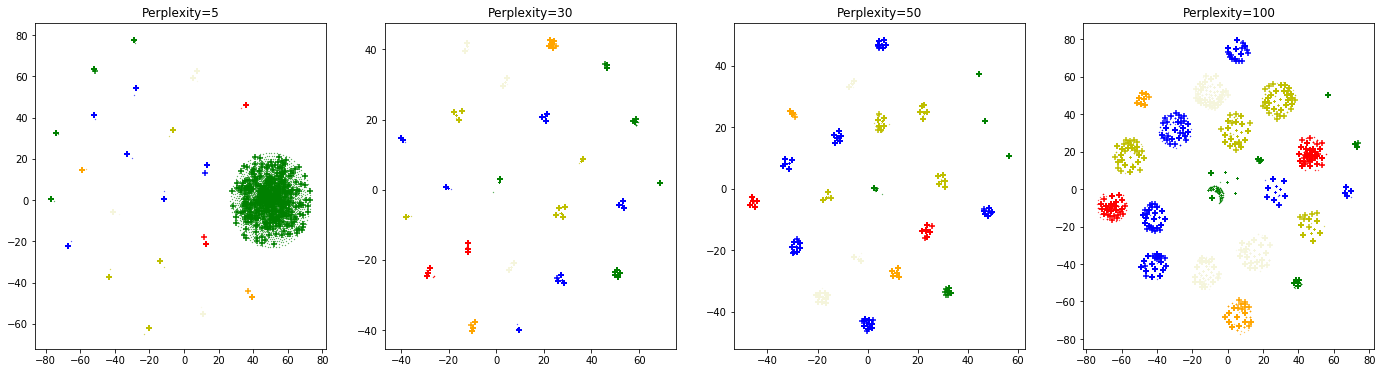}\label{fig:taxi_tsne}}
\subfigure[Bike with S2U]{\includegraphics[width = 0.3\linewidth, trim={37cm 0cm 0cm 0cm}, clip]{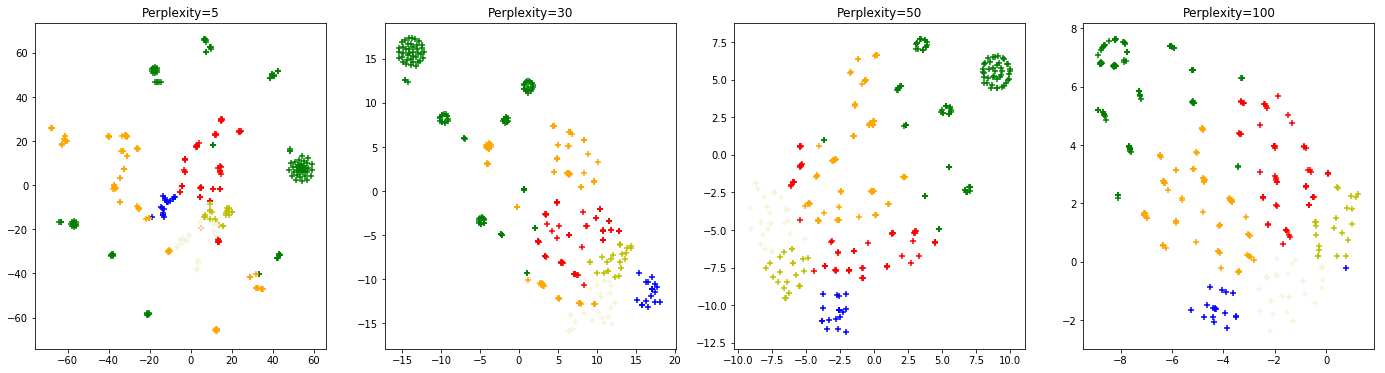}\label{fig:bike_tsne}}
\vspace{-1\baselineskip}

\caption{t-SNE visualization in which points are t-SNE transformed users using both raw and S2U-transformed users, and points' colors are based on Naive model and S2U model using 6 clusters.}
\label{fig:tsne}
\end{figure}

\textbf{Results of semantic-based aggregation of users:}
Two visualizations in Figure \ref{fig:user_clusters} demonstrate a qualitative performance of explainable clustering results based on S2U's raw weights and also semantic-based aggregated weights. In both heatmaps of sub-figures, each row of the heatmap represents a user, and each column is a weight for corresponding temporal signatures. The darker the blue is, the larger a weight value is, whose value can be found in the right-side color bar. Different colors in left-side color bars are cluster labels of different user groups. The $x$ axis is noted by the same index of its associated temporal signatures. The right sub-figure \ref{fig:user_group_aggregate} is noted by its combined temporal signatures' index, like $weight\_1+2$ means this weight is got by semantic-based aggregating temporal signatures $1$ and $2$. Both sub-figures have a strong similarity between a user's weights of temporal signatures and other users in the same clusters. Based on this explainable clustering results, we can interpret human daily life in Metro data. For example in sub-figure \ref{fig:user_group_normal}, the brown cluster (the second group from top) are users who have early working hour $7am$ and early back-home hour around $3pm$. the pink cluster (the third group from top) contains users who have later working hour $8am$ and later back-home hour $5pm$. We can use explainable clustering to support many such analysis of users.
\begin{figure}[htbp]
\centering
\vspace{-1\baselineskip}
\subfigure[S2U raw weights]{\includegraphics[width = 0.5\linewidth, trim={4.5cm 1cm 1cm 5cm}, clip]{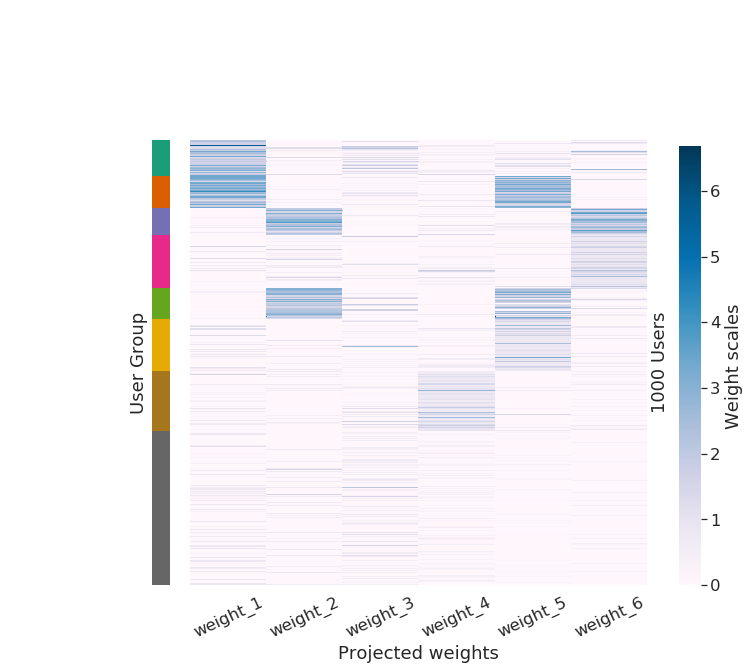}\label{fig:user_group_normal}}
\subfigure[Semantic-based aggregated weights]{\includegraphics[width = 0.49\linewidth, trim={4.5cm 1cm 1cm 5cm}, clip]{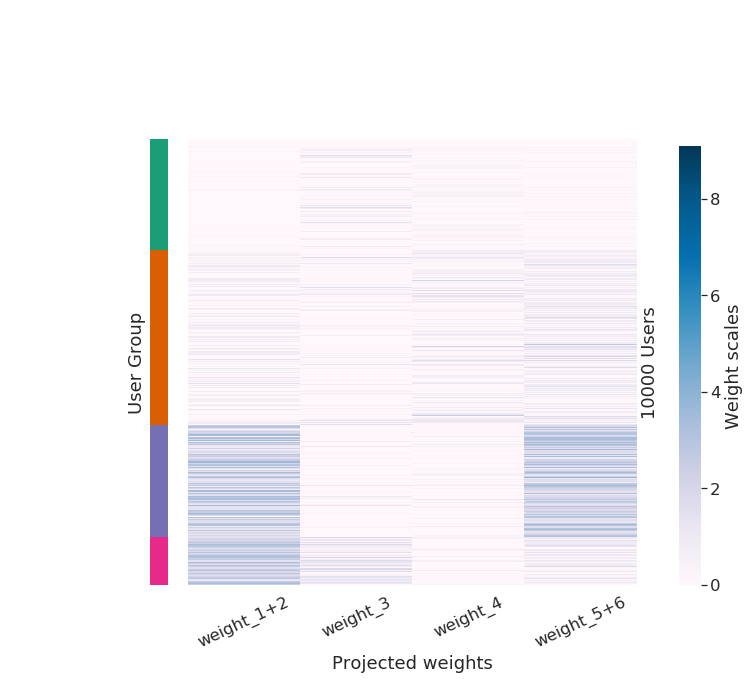}\label{fig:user_group_aggregate}}
\vspace{-1\baselineskip}

\caption{Explainable User Clustering Labels (6 clusters for for flows before aggregation and 4 clusters for flows after aggregation)}
\label{fig:user_clusters}
\end{figure}

\textbf{Results of generated semantic-based station pattern:}
Using the tidal-regularized loss, the following visualization in Figure \ref{fig:generativity_attractivity} shows that our TR-NMF could support more explainable station patterns combined with semantic-based aggregation. Both sub-figures demonstrate station locations (circles) on the area around The White House (the background map). The bigger the size of a circle, and the lighter a blue color is, the stronger attractivity pattern (recovered commuting in-flow for associated temporal signatures) is found. The black solid line is the Metro transit lines. The left sub-figure is based on the $7am$ commuting signature. We can see that station around The White House (mostly Federal Government offices) have more commuting flow which indicate a early working hour. In right sub-figure b), the station at Dupont Circle (a concentrated commercial area) has a larger $8am$ commuting in-flow, and stations around The White House decrease a little but still very strong. This example clearly illustrates how our S2U with TR-NMF can support explainable station patterns generation and intra-city function analysis.
\begin{figure}
\centering
\vspace{-1\baselineskip}
    \centering
    \includegraphics[width=\linewidth, trim={0cm 0cm 3cm 3cm}, clip]{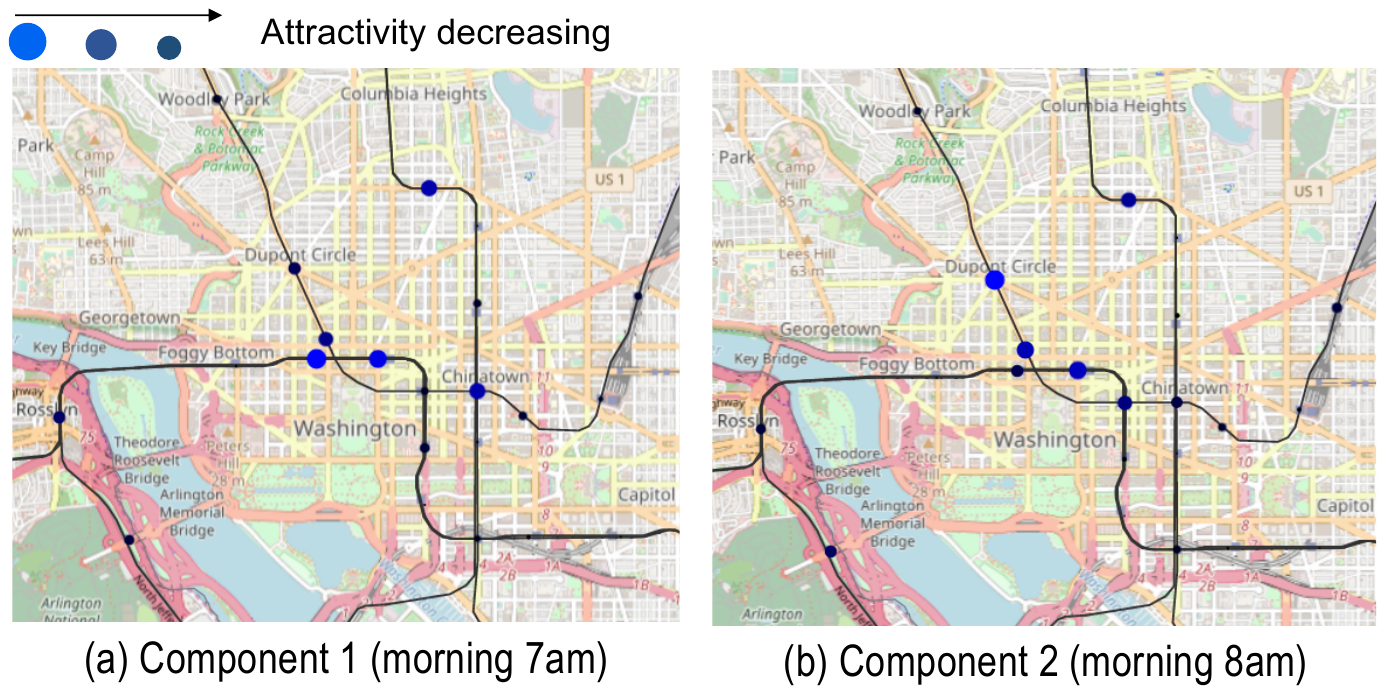}
    \vspace{-1\baselineskip}
    \caption{Explainable stations pattern around White House at Washington D.C.}
    \label{fig:generativity_attractivity}
\end{figure}



\section{Conclusion}
\label{sec:conclusion}
We propose a new Station-to-User (S2U) transfer learning framework to achieve a more explainable and stable learning of users in semantic-poor farecard data through transfering users to a latent feature space built with stations' temporal signatures. Also, we develop a novel Tidal-regularized Non-negative Matrix Factorization to guide the learning process towards a tidal-traffic commuting pattern that dominate urban transportation. To demonstrate the effectiveness of our work, we also set up a first-of-its-kind user stability test as a benchmarking evaluation metric to promote cross-model performance comparison. Lastly, we show that our framework improve $0.15$ for mean ARI and $0.21$ for median ARI in Taxi data experiment, and smaller margin of improvement for Metro and Bike data experiments. With visualization of t-SNE, we discuss the observations that reveal the power of S2U framework, and the difficulty of clustering tasks for Metro and Bike users. Finaly, we showcase how our explainable framework could support user behaviors analysis and station patterns analysis.

%
%
%
\bibliographystyle{splncs04}
\bibliography{main.bib}
%




\end{document}